\documentclass[preprint,12pt,authoryear]{elsarticle}

\usepackage{amsmath}
\usepackage{amssymb}
\usepackage{amsthm}
\usepackage{longtable}

\usepackage{booktabs}
\usepackage{natbib}

\usepackage{algorithm}
\usepackage{algorithmic}
\usepackage{bbm}

\usepackage{color}
\usepackage[normalem]{ulem}
\usepackage{fancyhdr}

\newtheorem{problem}{Problem}
\newtheorem{corollary}{Corollary}
\newtheorem{theorem}{Theorem}

\newtheorem{lemma}{Lemma}

\newcommand{\argmax}{\mathop{\rm argmax}\limits}

\makeatletter
\long\def\pprintMaketitlex{\clearpage
  \iflongmktitle\if@twocolumn\let\columnwidth=\textwidth\fi\fi
  \resetTitleCounters
  \def\baselinestretch{1}%
  \printFirstPageNotes
  \begin{\elsarticletitlealign}%
 \thispagestyle{pprintTitle}%
   \def\baselinestretch{1}%
    \Large\@title\par\vskip18pt%
    \ifx\@elsarticlenewpageafter\newpage@after@title
      \newpage
    \fi%
    \ifdoubleblind
      \vspace*{2pc}
    \else
      \normalsize\elsauthors\par\vskip10pt
      \footnotesize\itshape\elsaddress\par\vskip36pt
    \fi
    \ifx\@elsarticlenewpageafter\newpage@after@author
      \newpage
    \fi%
    \ifx\@elsarticlenewpageafter\newpage@after@abstract
      \newpage
    \fi%
    \end{\elsarticletitlealign}%
    \gdef\thefootnote{\arabic{footnote}}%
}
\def\ps@pprintTitle{%
     \let\@oddhead\@empty
     \let\@evenhead\@empty
     \def\@oddfoot
       {\hbox to \textwidth%
        {\ifnopreprintline\relax\else
        \@myfooterfont%
         \ifx\@elsarticlemyfooteralign\@elsarticlemyfooteraligncenter%
           \hfil\@elsarticlemyfooter\hfil%
         \else%
         \ifx\@elsarticlemyfooteralign\@elsarticlemyfooteralignleft%
           \@elsarticlemyfooter\hfill{}%
         \else%
         \ifx\@elsarticlemyfooteralign\@elsarticlemyfooteralignright%
           {}\hfill\@elsarticlemyfooter%
         \else%
         \iffalse  
               Preprint submitted to \ifx\@journal\@empty%
               Elsevier%
         \fi
         \else\@journal\fi\hfill\@date\fi%
         \fi%
         \fi%
         \fi%
         }
       }%
     \let\@evenfoot\@oddfoot}
\makeatother

\begin{document}
\begin{frontmatter}



\title{Gaussian Process Classification Bandits}


\author[H]{Tatsuya Hayashi}
\author[H]{Naoki Ito\fnref{nicurad}}
\fntext[nicurad]{Present Address: NTT DATA Corporation\\Toyosu Center Building, 3-3, Toyosu 3-chome, Koto-ku, 135-6033, Tokyo, Japan}
\author[R]{Koji Tabata}
\author[H]{Atsuyoshi Nakamura\corref{cor1}}
\ead{atsu@ist.hokudai.ac.jp}
\cortext[cor1]{Corresponding author}
\author[O]{Katsumasa Fujita}
\author[K]{Yoshinori Harada}
\author[R]{Tamiki Komatsuzaki}


\affiliation[H]{organization={Graduate School of Information Science and Technology, Hokkaido University},
            addressline={\\Kita 14, Nishi 9, Kita-ku}, 
            city={Sapporo},
            postcode={060-0814}, 
            state={Hokkaido},
            country={Japan}}

\affiliation[R]{organization={Research Institute for Electronic Science, Hokkaido University},
            addressline={\\Kita 10, Nishi 8, Kita-ku}, 
            city={Sapporo},
            postcode={060-0810}, 
            state={Hokkaido},
            country={Japan}}

\affiliation[O]{organization={Department of Applied Physics, Osaka University},
            addressline={\\2-1 Yamadaoka}, 
            city={Suita},
            postcode={565-0871}, 
            state={Osaka},
            country={Japan}}

\affiliation[K]{organization={Graduate School of Medical Science, Kyoto Prefectural University of Medicine},
            addressline={\\465 Kajii-cho}, 
            city={Kamigyo-ku},
            postcode={602-8566}, 
            state={Kyoto},
            country={Japan}}
\begin{abstract}
Classification bandits are multi-armed bandit problems whose task is to classify a given set of arms into either positive or negative class
depending on whether the rate of the arms with the expected reward of at least $h$ is not less than $w$ for given thresholds $h$ and $w$.
We study a special classification bandit problem in which arms correspond to points $x$ in $d$-dimensional real space with expected rewards $f(x)$ which are generated according to a Gaussian process prior. We develop a framework algorithm for the problem using various arm selection policies 
and propose policies called FCB and FTSV.
We show a smaller sample complexity upper bound for FCB than that for the existing algorithm of the level set estimation, in which whether $f(x)$ is at least $h$ or not must be decided for every arm's $x$. 
Arm selection policies depending on an estimated rate of arms with rewards of at least $h$ are also proposed and shown to 
improve empirical sample complexity.
According to our experimental results, the rate-estimation versions of FCB and FTSV, together with that of the popular active learning policy that selects the point with the maximum variance, outperform other policies for synthetic functions, and
the version of FTSV is also the best performer for our real-world dataset.
\end{abstract}


\end{frontmatter}


\section{Introduction}
Classification is a task that classifies a given instance into one of the pre-defined classes.
Anomaly detection from image data can be seen as a classification task in which the data is classified into normal and anomaly.
In some anomaly detection such as disease diagnoses, anomaly class can be defined by bad area ratio: the data is an anomaly if and only if the bad area ratio of the data is at least a certain threshold. In some cases like disease diagnoses by Raman spectra \citep{khalifa2019}, we have to measure the badness of each point in a given region one by one
instead of a one-shot measurement of the whole region.
In this paper, we study the fast classification of a given region by its bad area ratio through measuring the badness of a point one by one iteratively. 

Let a point $x$ in a given region of $d$-dimensional real space $\mathbb{R}^d$ be bad if and only if the badness $f(x)$ of $x$ is at least a given threshold, and assume that we can know noisy $f(x)$ by measuring at point $x$.
We also assume that the bad area ratio can be approximated by the bad point ratio in the finite point set $D$ that is composed of (coarsely) quantized points in the given bounded region $X\subset \mathbb{R}^d$. 
To introduce correlation between $f(x)$ of different points, $f(x)$ for all $x\in D$ are assumed to be generated by the Gaussian process prior with mean $0$ and some appropriate kernel.
Correlation brought by the Gaussian process assumption makes correct judgment possible with
high probability by measuring only at a part of the whole point set.
The level set estimation problem studied by \citet{LSE} is a problem of this setting though the objective is the discrimination of 
whether each point in the whole set is bad or not.
In our classification setting, we do not have to know the badness of every point in the whole set to judge whether the bad area ratio is at least a given threshold or not, and thus an earlier stopping time can be realized.

Since we only consider a finite point set $D$ in a given region, the problem we deal with here is just a multi-armed classification bandit problem studied by \citet{Tabata2021}.
In the paper, however, no correlation between arms' rewards  is assumed.
By making use of the correlation between points' rewards in our problem setting, a reduction of the sample complexity can be expected.
The result of our experiment using a real-world dataset shows the effectiveness of using the correlation between points' rewards.

We propose Algorithm GPCB[ASP] (Gaussian Process Classification Bandits [Arm Selection Policy]),
in which any arm selection policy function ASP can be used.
We prove the correctness of GPCB[ASP] for any ASP as a solution for our PAC-like\footnote{PAC is an abbreviation for Probably Approximately Correct.} problem: for given $\delta,\varepsilon>0$ and given thresholds $h, w$,
 decide whether
the rate of points with badness at least $h+\varepsilon$ is at least $w$ or
 the rate of points with badness less than $h-\varepsilon$ is more than $1-w$, correctly with probability at least $1-\delta$ when one of them is true.
 For the PAC-like problem, we show a high-probability sample complexity upper bound $C/(\Delta_w+\varepsilon)^{2+o(1)}$ of GPCB[FCB] using a specific ASP called FCB (Farthest Confidence Bound), where $C$ is a constant value that depends on the dimension $d$ and the adopted kernel in Gaussian process, and $\Delta_w$ is the difference between $\lceil w|D|\rceil$th largest badness value $f(x)$ $(x\in D)$ and threshold $h$.
 This upper bound is smaller than the sample complexity upper bound $C/(2\varepsilon)^{2+o(1)}$ of the LSE algorithm
 for the harder level set estimation problem in the case with $\Delta_w>\varepsilon$. 
 Furthermore, we prove that GPCB[ASP] which is modified for the level set estimation and uses an ASP of always selecting an uncertain arm, achieves the same sample complexity upper bound $C/(2\varepsilon)^{2+o(1)}$ as the LSE algorithm,
 which implies the same sample complexity upper bound of GPCB[ASP] for the easier classification bandits.
As arm selection policy functions, we also propose \emph{rate estimation} versions of ASPs that use the estimation $\hat{r}$ of rate $|H_h|/|D|$, and narrow the candidates to $\tilde{H}$
if $\hat{r}\geq w$ and $D\setminus\tilde{H}$ otherwise, where $H_h$ is a set of points whose badness is at least $h$, and $\tilde{H}$ is
the estimated set of $H_h$.

We check empirical sample complexities of GPCB[ASP] for various ASPs, including their rate estimation versions using two types of synthetic functions (GP-prior-generated functions and benchmark functions) and a real-world dataset. 
For synthetic functions, GPCB[ASP]s with ASP$=$FCB and FTSV (Farthest Thompson Sampling Value) outperform those with LSE (Level Set Estimation) \citep{LSE} and STR (STRaddle heuristic) \citep{Bryan2006}, where LSE and STR are state-of-the-art arm selection policies used for the level set estimation.
Performance of the rate estimation version ASP$_{\text{RE}}$ of an ASP significantly improves for FCB, FTSV, and VAR (VARiance), where VAR is a popular active learning policy that selects the point with the maximum variance.  As a result, GPCB[FCB$_{\text{RE}}$], GPCB[FTSV$_{\text{RE}}$], and GPCB[VAR$_{\text{RE}}$] are the best performers for  both types of synthetic functions.
As the real-world dataset, we use cancer index images that are made from Raman images of cancer and non-cancer thyroid follicular cells.
Our algorithm GPCB[ASP] correctly classifies all nine images into cancer  and non-cancer ones by iteratively asking cancer indices of less than 0.95\% points for all the ASPs we used. In contrast, cancer indices of at least 1.6\% points are required for the existing classification bandit algorithm \citep{Tabata2021}.
GPCB[FTSV$_{\text{RE}}$] is the best performer for this real-world dataset.

\subsection*{Related Work}

This research is categorized into stochastic bandit problem \citep{auer2002finite}
of pure exploration \citep{bubeck2009} with fixed confidence setting \citep{even-dar2002}.
Both exploration and exploitation are needed for the objective of cumulative reward maximization,
but exploration only is needed in the case that algorithms are evaluated by their final outputs only, and such kinds of problems are called pure exploration problems.
The most popular pure exploration bandit problem is the best arm identification \citep{audibert2010}, and there is its variation that identifies all the \emph{above-threshold arms}, that is, the arms with a mean reward more than a given threshold \citep{Locatelli,kano}.
For a given reward threshold, instead of identifying all the above-threshold arms, the decision bandit problem of whether or not an above-threshold arm exists,  was
studied in \citep{Kaufmann2018,Tabata2020}.
Recently, an extended decision bandit problem called classification bandits has been proposed \citep{Tabata2021}.
In classification bandits, a threshold on the number of above-threshold arms is introduced, and the problem is to decide whether the number of above-threshold arms
in a given set of arms is more than the introduced threshold or not.
Classification bandit problem can be seen as a kind of more general active sequential hypothesis testing problem,
which was studied and nice theoretical analyses were provided by \citet{DK2019,garivier2021}.
In this paper, we study the classification bandit problem with finite arms
that correspond to points in $\mathbb{R}^d$ with correlated rewards depending on their positions.

In the context of active learning, the problem of level set estimation is studied.
The level set estimation is a problem to determine whether an unknown function value at $x$ is above or below some given threshold for all the points $x$ in a given set $D\in \mathbb{R}^d$
by iteratively asking noisy function values.
In \citep{Bryan2006,LSE}, the Gaussian process prior is used as a correlation assumption for function values,
and \citet{LSE} proved the sample complexity upper bound of their proposed algorithm LSE and reported the F1-scores of LSE and STR better than that of VAR (VARiance),
where STR is a straddle heuristic proposed by \citet{Bryan2006} and VAR is a popular active learning method that selects the maximum variance points.
The problem we deal with can be seen as an easier variation of the level set estimation in which it is only checked whether the rate of points with above-threshold function value is above a given threshold or not. We prove a better sample complexity upper bound for this variation.

\section{Problem Settings}

For a bounded set $X\subset \mathbb{R}^d$, assume the existence of an unknown function $f : X \to \mathbb{R}$.
What we want to know is whether the ratio of the area $f(x)\geq h$ is at least $w$ or not for given two thresholds $h,w\in \mathbb{R}$.
Here, we assume that the ratio can be approximately estimated by the ratio of those in a finite set $D$ of points in $X$,
and formulate the problem for $D\subset X\subset \mathbb{R}^d$.
Observable information for answering the question is noisy function values $y=f(x)+\eta$ for requested points $x\in D$, where $\eta\in \mathbb{R}$ is assumed to be generated according to a normal distribution $N(0,\sigma^2)$.
Our objective is to answer the question correctly with high probability by asking as a small number of noisy function values
as possible.
In this setting, it is difficult to correctly answer whether $f(x)\geq h$ or not for the point $x$ whose function value is close to the threshold $h$. Thus, by introducing a margin $\varepsilon$, we consider the problem of
whether the ratio of the area $f(x)\geq h+\varepsilon$ is at least $w$ or not.
We formally define this problem as follows.

\begin{problem}[Classification Bandits with Margin $\epsilon$]\label{prob:CBwM}
  Given $\varepsilon>0$ and $\delta>0$, output ``positive'' with probability at least $1-\delta$
  if $|H_{h+\varepsilon}|/|D|\geq w$
    and output ``negative'' if $|L_{h-\varepsilon}|/|D|> 1-w$
      by asking as small number of noisy function values of $f$ as possible,
      where $H_{h+\varepsilon}=\{x\in D\mid f(x)\geq h+\varepsilon\}$, $L_{h-\varepsilon}=\{x\in D\mid f(x)< h-\varepsilon\}$, and $|\cdot|$ denotes the number of elements in a set `$\cdot$'.
\end{problem}

Note that any output with any probability is allowed when $|H_{h+\varepsilon}|/|D|< w$ and $|L_{h-\varepsilon}|/|D|\leq 1-w$.

This bandit problem is a  pure exploration problem like the best arm identification.
This problem can be seen as a classification problem in which a set $D$ with a function $f$ is classified into ``positive'' or ``negative''. Thus, this kind of bandit problem is called \textit{classification bandits} \citep{Tabata2021}.
In this paper, we develop faster classification algorithms by assuming a correlation between  $f(x)$ of points in $D\subset X$. 

As the \emph{level set estimation} that \citet{LSE} studied,
we assume that target function values $f(x)\in D$ are generated from a Gaussian process (GP),
that is, 
a target function $f(\cdot)$ over $D$ is drawn according to $\mathrm{GP}(\mu(\cdot), k(\cdot,\cdot))$
with $\mu(x)=0$ for $x\in D$ and some kernel function $k(\cdot,\cdot)$.
Throughout this paper, we assume $k(x,x')\leq 1$ for $x,x'\in D$.
Given a list of $t$ noisy observations $\mathbf{y}_t = (y_1, \dots, y_t)^\mathsf{T}$ for a list of $t$ inputs $(x_1,\dots,x_t)$,
the posterior of $f(\cdot)$ is $\mathrm{GP}(\mu_t(\cdot),k_t(\cdot,\cdot))$ with $\mu_t$ and $k_t$ calculated as follows:
\begin{align}
\mu_t(x) &= {\mathbf{k}}_t(x)^\mathsf{T}({\mathbf{K}}_t+\sigma^2{\mathbf{I}})^{-1}\mathbf{y}_t, \label{mu_t}\\
k_t(x,x^\prime) &=k(x,x^\prime)-{\mathbf{k}}_t(x)^T({\mathbf{K}}_t+\sigma^2{\mathbf{I}})^{-1}{\mathbf{k}}_t(x), \text{ and} \label{k_t}\\
\sigma_t^2(x) &= k_t(x, x),
\end{align}
where ${\mathbf{k}}_t(x)=[k(x_1,x),\dots,k(x_t,x)]^\mathsf{T}$ and ${\mathbf{K}}_t=[\mathbf{k}_t(x_1),\dots,\mathbf{k}_t(x_t)]$.
What we call \emph{Gaussian process classification bandits} is Problem~\ref{prob:CBwM}  with this Gaussian process assumption.

The Gaussian process classification bandit problem is a problem easier than the level set estimation; in level set estimation, all the points in $H_{h+\varepsilon}\cup L_{h-\varepsilon}$ must be identified correctly with probability at least $1-\delta$, from which the answer of Problem~\ref{prob:CBwM} can be derived. For this easier problem, we show arm selection policies with
smaller sample complexities theoretically and empirically.

\section{Proposed Method}

\subsection{Algorithm GPCB[ASP]}

\begin{algorithm}[h!]
  \caption{GPCB[$\mathrm{ASP}$] ($D,k,\varepsilon,\delta,h,w$)}\label{cbca}
\begin{algorithmic}[1]
\renewcommand{\algorithmicrequire}{\textbf{Parameter:}}
	  \REQUIRE $\mathrm{ASP}$: Arm Selection Policy
\renewcommand{\algorithmicrequire}{\textbf{Input:}}
          \REQUIRE
          \begin{minipage}[t]{10cm}
          $D$: samplable point set,\\
          $k$: kernel of the GP prior,\\
          $\varepsilon$: accuracy parameter,\\
          $\delta$: confidence parameter,\\
          $h$: function value threshold,\\
          $w$: high-value ratio threshold
          \end{minipage}
 \renewcommand{\algorithmicrequire}{\textbf{function:}}
  \REQUIRE $\beta_t=2\log(|D|\pi^2t^2/6\delta)$
  \renewcommand{\algorithmicensure}{\textbf{Output:}}
          \ENSURE ``positive'' or ``negative''
		     \STATE $\hat{H}\gets \phi, \hat{L}\gets\phi, U_0\gets D$
                \STATE $C_0(x)\gets (-\infty,\infty), \mu_0(\cdot)\gets 0, k_0(\cdot,\cdot)\gets k(\cdot,\cdot)$
		\STATE $t\gets 1$
		\WHILE {$U_{t-1}\neq\phi$}
                \STATE $U_t\gets \emptyset$
		\FOR{all $x\in U_{t-1}$}
		\STATE $C_t(x)\gets C_{t-1}(x)\cap[\mu_{t-1}(x)-\sqrt{\beta_t}\sigma_{t-1}(x),\mu_{t-1}(x)\!+\!\sqrt{\beta_t}\sigma_{t-1}(x)]$
		\IF{$\min (C_t(x))\geq h-\varepsilon$}
		\STATE $\hat{H}\gets \hat{H}\cup \{x\}$\label{H_t}
		\ELSIF{$\max (C_t(x))< h+\varepsilon$}
		\STATE $\hat{L}\gets \hat{L}\cup \{x\}$\label{L_t}
                \ELSE
                \STATE $U_t\gets U_t\cup \{x\}$\label{U_t}
		\ENDIF
		\ENDFOR
		\IF {$|\hat{H}|/|D|\geq w$}\label{alg:stcheckcbb}
		\RETURN{``positive''}
		\ELSIF{$|\hat{L}|/|D|> 1-w$}
		\RETURN{``negative''}
		\ENDIF\label{alg:stcheckcbe}
		\STATE $x_t\gets \mathrm{ASP}(U_t,C_t,\mu_{t-1},k_{t-1},h)$\label{asp}
		\STATE Ask $f(x_t)$ and receive its noisy value $y_t$  
		\STATE Calculate $\mu_t$ and $k_t$ using Eqs. (\ref{mu_t}) and (\ref{k_t})
		\STATE $t\gets t+1$
		\ENDWHILE
	\end{algorithmic}
\end{algorithm}

For Gaussian process classification bandits, we propose an algorithm GPCB[ASP] (Gaussian Process Classification Bandits [Arm Selection Policy]) using any arm selection policy function ASP, which is shown in Algorithm~\ref{cbca}.
Function ASP, which is used in Line~\ref{asp}, selects the point $x_t\in D$ to ask its noisy function value $y_t$ at time step $t$.
GPCB[ASP] maintains a high-probability confidence interval $C_t(x)$ for each point $x\in D$ at each time step $t$ and puts $x$ into the
estimated high-value point set $\hat{H}$ if $\min C_t\geq h-\varepsilon$ (Line~\ref{H_t}), into the estimated low-value point set $\hat{L}$ if $\max C_t<h+\varepsilon$ (Line~\ref{L_t}), and into the uncertain point set $U_t$ otherwise (Line~\ref{U_t}).
GPCB[ASP] stops when $|\hat{H}|/|D|\geq w$ or $|\hat{L}|/|D|> 1-w$ holds, and outputs ``positive'' in the former case  and ``negative'' in the latter case.

Differences from the LSE algorithm \citep{LSE} are the following two points:(1) in addition to condition $U_{t-1}= \emptyset$, two other stopping conditions $|\hat{H}|/|D|\geq w$ and $|\hat{L}|/|D|> 1-w$ are checked for the realization of earlier stopping time in our easier problem setting, (2) any ASP can be used instead of fixing as ASP$=$LSE.

As the LSE algorithm \citep{LSE} for the level set estimation,
we adopt
\begin{equation}
C_t(x)=\bigcap_{i=1}^t[\mu_{i-1}(x)-\sqrt{\beta_i}\sigma_{i-1}(x),\mu_{i-1}(x)+\sqrt{\beta_i}\sigma_{i-1}(x)] \label{def:C_t}
\end{equation}
as the confidence interval $C_t(x)$ for the GP posterior $\mathrm{GP}(\mu_t(\cdot),k_t(\cdot,\cdot))$ using the GP prior $\mathrm{GP}(0,k(\cdot,\cdot))$,
where $\beta_t$, $\mu_0(x)$ and $\sigma_0(x)$ are set as $\beta_t=2\log(|D|\pi^2t^2/6\delta)$, $\mu_0(x)=0$ and $\sigma_0(x)=\sqrt{k(x,x)}$.

Using confidence intervals $C_t(x)$ defined by Eq.~(\ref{def:C_t}), the following corollary is known to hold. The corollary leads to Theorem~\ref{th:GPCB}
 which guarantees the correct output of GPCB[ASP] for any arm selection policy if it terminates.

\begin{corollary}[Corollary 1 in \citep{LSE}]\label{APcorrect-assumptionen}
  
  Let $\beta_t=2\log(|D|\pi^2t^2/6\delta)$ for $ \delta\in(0,1) $ and $t\geq 1$. Then,
	$$ f(x)\in C_t(x) \text{\ for all $x\in D$ and $t\geq 1$}$$
	holds with probability at least $1- \delta $.
\end{corollary}

\begin{theorem}\label{th:GPCB}
	For any $h\in \mathbb{R}, w,\delta\in(0,1), \epsilon>0$, 
	if algorithm GPCB[ASP] terminates, it outputs ``positive'' with probability at least $1-\delta$ in the case with $|H_{h+\epsilon}|/|D|\geq w$,
	and ``negative'' with probability at least $1-\delta$ in the case with $|L_{h-\epsilon}|/|D|> 1-w$.
\end{theorem}
\begin{proof}
  If GPCB[ASP] terminates at time $t$, $|\hat{H}|/|D|\geq w$ or $|\hat{L}|/|D|> 1-w$ must hold at time $t$ because $(|\hat{H}|+|\hat{L}|)/|D|<1$ holds otherwise,
  which implies $U_t\neq \emptyset$ that contradicts the fact GPCB[ASP] terminates at time $t$.
  Thus, whenever GPCB[ASP] terminates, it outputs ``positive'' or ``negative''. Let $t$ be the stopping time.
  
	For $C_t(x)$ defined by Eq.~(\ref{def:C_t}), by Corollary~\ref{APcorrect-assumptionen},
	\begin{align}
	f(x)\in C_t(x) \text{ for any } x\in D \text{ and } t\geq 1 \label{CIcorrect-assumption}
	\end{align}
	holds with probability at least $1-\delta$.
	If (\ref{CIcorrect-assumption}) holds, GPCB[ASP] never outputs ``negative'' in the case with $|H_{h+\epsilon}|/|D|\geq w$ because,
	for all $x\in H_{h+\epsilon}$ and $t\geq 1$, $\max(C_t(x))\geq f(x)\geq h+\varepsilon$ holds, and thus $x\not\in \hat{L}$.
	Thus, GPCB[ASP] always outputs ``positive'' in this case if it terminates. We can similarly prove the theorem statement in the case with $|L_{h-\epsilon}|/|D|> 1-w$. 
\end{proof}

\subsection{Arm Selection Policies for Algorithm GPCB}

We consider three arm selection policy functions, FCB, LSE, and FTSV.
In each of them, we define value $a_t(x)$ for each point $x\in U_t$ and $t\geq 1$, which indicates how good the point $x$ is as the next point to ask its noisy function value of $f(x)$. Using different $a_t(x)$, each of them selects $x\in U_t$ with the maximum $a_t(x)$ as the next query point $x_t$, that is,
\begin{equation}
 \text{ASP}(U_t,C_t,\mu_{t-1},k_{t-1},h)=\argmax_{x\in U_t}a_t(x). \label{ASP:argmax}
\end{equation}
The definitions of $a_t(x)$ for each arm selection policy function ASP are as follows:
\begin{description}
\item[FCB](Farthest Confidence Bound)
  \[
a_t(x)=\max\{\max(C_t(x))-h,h-\min(C_t(x))\}
\]
\item[LSE](Level Set Estimation \citep{LSE})
  \[
a_t(x)=\min\{\max(C_t(x))-h,h-\min(C_t(x))\}
\]
\item[FTSV](Farthest Thompson Sampling Value)
  \[
  a_t(x)=|g_t(x)-h| \text{ for } g_t\sim \mathrm{GP}(\mu_{t-1},k_{t-1})
  \]
\end{description}

Algorithm GPCB[ASP] should select a point $x_t$ that promotes the growth of the estimated high-value point set $\hat{H}$ if the ``positive'' possibility is high and promotes the growth of the estimated low-value point set $\hat{L}$ otherwise.
Incorporating this idea to the arm selection policies, we also propose their rate estimation versions FCB$_{\text{RE}}$, LSE$_{\text{RE}}$,  and FTSV$_{\text{RE}}$
as follows:
\begin{align}
\text{ASP}(U_t,C_t,\mu_{t-1},k_{t-1},h)
= \begin{cases} \argmax_{x\in \tilde{H}_t}a_t(x) & \text{ if } |\hat{H}\cup \tilde{H}_t|/|D|\geq w\\
  \argmax_{x\in U_t\setminus\tilde{H}_t}a_t(x) & \text{ otherwise, }  
  \end{cases} \label{ASP:rateestimation}
\end{align}
where $\tilde{H}_t$ is defined as
\[
\tilde{H}_t=\begin{cases}
\{x\in U_t\mid \max(C_t(x))-h\geq h-\min(C_t(x))\} & \text{ (FCB$_{\text{RE}}$, LSE$_{\text{RE}}$)}\\
\{x\in U_t\mid g_t(x)\geq h\} & \text{ (FTSV$_{\text{RE}}$).}
\end{cases}
\]

\subsection{Sample Complexity Upper Bounds}

 This section first shows a sample complexity upper bound of algorithm GPCB[ASP] for any arm selection policy ASP that always selects a point $x\in U_t$. Then we also show a high-probability sample complexity upper bound of algorithm GPCB[FCB] that is better than it for  problem instances with some moderate complexity.

As the sample complexity upper bound of Algorithm LSE for the level set estimation in \citep{LSE},
those sample complexity upper bounds can be derived by lower and upper bounding the information gain $I(\mathbf{y}_t;\mathbf{f}_t)$ from a list of noisy function values $\mathbf{y}_t=(y_1,\dots,y_t)^\mathsf{T}$
for a list of target function values $\mathbf{f}_t=(f(x_1),\dots,f(x_t))^\mathsf{T}$ of $\mathbf{x}_t=(x_1,\dots,x_t)^\mathsf{T}$.
Information gain $I(\mathbf{y}_t;\mathbf{f}_t)$ is defined as
\[
I(\mathbf{y}_t ; \mathbf{f}_t) = H(\mathbf{y}_t) - H(\mathbf{y}_t \mid \mathbf{f}_t).
\]
Define $\gamma_t$ as
\[
\gamma_t =  \max_{\mathbf{x_t}\in D} I(\mathbf{y}_t; \mathbf{f}_t).
\]
Then, $I(\mathbf{y}_t; \mathbf{f}_t)$ is upper bounded by $\gamma_t$, and its increasing order is known to be
$\gamma_t=O(d\log t)$ for linear kernel $k(x,x')=x^\mathsf{T}x'$, $\gamma_t=O((\log t)^{d+1})$ for the squared exponential kernel $e^{-\frac{||x-x'||^2}{2\ell}}$ with length scale parameter $\ell$, where $x,x'\in D\subset \mathbb{R}^d$ \citep{gpucb}.

As a complexity measure of problem instances, we use $\Delta_w$ which is defined as 
\[
\Delta_w = | f(x^{(\lceil w|D|\rceil )}) -h |,
\]
where $x^{(i)}$ is the point that has the $i$-th largest function value in $D$. The complexity of the problem increases as $\Delta_w$ becomes smaller.
We evaluate sample complexity upper bound in the form of $T(\cdot)$, 
where function $T(\Delta)$ is defined as
\[
T(\Delta)=\min\left\{t\in \mathbb{N}  \left| \ \frac{t}{\beta_t\gamma_t}> \frac{C_1}{\Delta^2}\right.\right\}.
\]
Here $C_1=8/\log(1+\sigma^{-2})$. Note that this function $T(\Delta)$ is a nonincreasing function, and thus $T(\Delta)$ for the larger $\Delta$ is better as sample complexity. We know $T(\Delta)<\frac{C}{\Delta^{2+o(1)}}$ for some constant $C$ that depends on $|D|$, $\delta$, $d$ and $\sigma$. (See \ref{sec:Tdelta}.) Furthermore, we can also derive  
$T(\Delta)<\left(\frac{C'd}{\Delta^2}\right)^{1+o(1)}$ for the linear kernel
and $T(\Delta)<\left(\frac{C''(d+2)^{(d+2)}}{\Delta^2}\right)^{1+o(1)}$ for the squared exponential kernel for some constants  $C'$ and $C''$ that depend on $|D|$, $\delta$, and $\sigma^2$.

The sample complexity of Algorithm GPCB[ASP] for any arm selection policy function ASP that always selects $x$ from the uncertain point set $U_t$ at time $t$ is upper bounded by $T(2\varepsilon)$ by the following theorem.

\begin{theorem}\label{coro:wcen}
	For any $h\in \mathbb{R},w,\delta\in(0,1)$, $\epsilon>0$, and for any arm selection policy function ASP that selects $x\in U_t$ at any time $t$,   
	Algorithm GPCB[ASP] always terminates after receiving at most $T(2\epsilon)$ function values.
\end{theorem}
\begin{proof}
See \ref{sec:coro:wcen}.
\end{proof}

The proof of Theorem~\ref{coro:wcen} is a simplified refinement of that of Theorem~1 in \citep{LSE},
and takes the stopping condition $U_t=\emptyset$ only into account. Thus the sample complexity bound $T(2\epsilon)$
holds even for the modified GPCB  in which checking of additional two conditions (Line~\ref{alg:stcheckcbb}-\ref{alg:stcheckcbe}) is removed. This fact means that, though \citet{LSE} proved the sample complexity bound $T(2\epsilon)$
for the arm selection policy LSE only, the same bound $T(2\epsilon)$ can be proved for any arm selection policy that selects $x\in U_t$ at any time $t$ in the problem of level set estimation.

When $\Delta_w>\varepsilon$, we can prove better sample complexity bound $T(\Delta_w+\varepsilon)$ depending on our problem setting
for Algorithm GPCB[FCB] as the following theorem.

\begin{theorem}\label{th:fcben}
	For any $h\in \mathbb{R}, w,\delta\in (0,1)$ and $\epsilon>0$, 
	Algorithm GPCB[FCB] terminates after receiving at most $T(\Delta_w+\epsilon)$ noisy function values with probability at least $1-\delta$.
\end{theorem}
\begin{proof}
See \ref{sec:SCUFCB}.
\end{proof}

\section{Experiments}
In this section, we compare the experimental performance of Algorithm GPCB[ASP] for various ASPs and their rate estimation versions using synthetic functions and real-world data. In all the result tables, the best numbers are bolded, and the numbers that are not significantly different from them are italicized.

\subsection{Comparison Algorithms}

In addition to ASP$=$FCB, LSE, FTSV, and their rate estimation versions, we compare their performances of GPCB[ASP] with those of the following existing ASPs,
which also select $x_t$ by Eq. (\ref{ASP:argmax}) using $a_t(x)$ defined as:
\begin{description}
\item[UCB] (Upper Confidence Bound)
  \[
  a_t(x)=\mu_{t-1}(x)+\sqrt{\beta_t}\sigma_{t-1}
  \]
\item[LCB] (Lower Confidence Bound)
  \[
  a_t(x)=\mu_{t-1}(x)-\sqrt{\beta_t}\sigma_{t-1}
  \]
\item[TS] (Thompson Sampling)
\[
  a_t(x)=g_t(x) \text{ for } g_t\sim \mathrm{GP}(\mu_{t-1},k_{t-1})
  \]
\item[VAR] (VARiance)
  \[
  a_t(x)=\sigma_{t-1}^2(x)=k_{t-1}(x,x)
  \]
\item[STR] (STRaddle heuristic \citep{Bryan2006})
\[
a_t(x)=1.96\sigma_{t-1}(x)-|\mu_{t-1}(x)-h|
\]
\end{description}
The rate estimation versions VAR$_{\text{RE}}$ and STR$_{\text{RE}}$ of VAR and STR, respectively,
select $x_t$ by Eq.~(\ref{ASP:rateestimation}) using $\tilde{H}_t$ defined as
\[
\tilde{H}_t=\{x\in U_t\mid \mu_{t-1}(x)\geq h\}.
\]
For every algorithm, we set $\sqrt{\beta_t} = 3$ as in the experiments by \citet{LSE}.  The dimensions $d$ of all the datasets are two. 
We use Gaussian kernel $k(x,x') = \sigma^2_{kernel} \exp (\| x-x'\|^2/2l^2)$ with parameters $\sigma^2_{kernel}$ and $l$ in all the experiments.
In all the experiments except that using GP-prior-geneated functions, we tune kernel parameters by maximum likelihood estimation for 100 samples generated from a uniform distribution.

\subsection{Synthetic Functions}

We conduct numerical experiments using two types of synthetic functions.
One is a type of function generated according to a GP prior and the other is a type of function used as a benchmark.
In all the experiments using synthetic functions, domain $D$ is set to $[a,a+b]^2$ for some $a,b\in \mathbb{R}$,
and $D$ is defined as $D=\{(a+ib/30,a+jb/30)\mid i,j=0,1,\dots,30\}$. 
Note that $|D|=961$.

\subsubsection{GP-Prior-Generated Functions}

\begin{table}[t]
\centering
  \caption{\#(function values) asked by GPCB[ASP] for ASPs \emph{without} rate estimations.
The numbers are averaged over 100 functions generated from $\mathrm{GP}(0,k(x,x'))$ and their 95\% confidence intervals are also shown.
The underlined bolded numbers  are the best among those for all the ASPs in Tables~\ref{ex:GPprior} and \ref{ex:GPprior2}.}\label{ex:GPprior}
  \begin{tabular}{lr@{$\pm$}rr@{$\pm$}rr@{$\pm$}rr@{$\pm$}r}
  \hline
    \hline
      ASP                       & \multicolumn{2}{c}{$w=0.2$} & \multicolumn{2}{c}{$w=0.4$} & \multicolumn{2}{c}{$w=0.6$} & \multicolumn{2}{c}{$w=0.8$}  \\ \hline
      FCB                       & 71.94 & 2.17 & 180.5 & 7.9 & 184.0 & 7.9 & 72.04 & 1.87   \\
      LSE                        & 72.09 & 1.76 & 199.0 & 16.6 & 202.4 & 15.9 & 71.70 & 1.50   \\
      FTSV      & 72.95 & 2.16 & 185.2 & 9.6 & 188.1 & 7.9 & 72.53 & 1.99\\
      STR                       & 73.00 & 0.55 & 261.5 & 2.9 & 386.0 & 26.9 & 84.00 & 1.94   \\
      VAR                       & 64.00 & 0.83 & \textit{174.0} & 10.3 & 216.0 & 4.4 & 79.00 & 1.66   \\
      TS               & \textbf{57.58} & 1.75 & \textbf{159.3} & 7.4 & 337.3 & 18.0 & 112.20 & 2.67  \\
      UCB                       & \textit{58.50} & 0.42 & 180.0 & 3.1 & 451.0 & 23.0 & 112.00 & 1.11  \\
      LCB                       & 101.44 & 2.41 & 301.5 & 21.6 & \underline{\textbf{155.1}} & 6.4 & \textbf{55.83} & 1.62  \\ \hline
      \hline
    \end{tabular}
\end{table}

\begin{figure}[t]
\centering
\begin{tabular}{@{}c@{}c@{}}
growth curve of $|\hat{H}|/|D|$ & growth curve of $|\hat{L}|/|D|$\\
\includegraphics[width=0.5\textwidth]{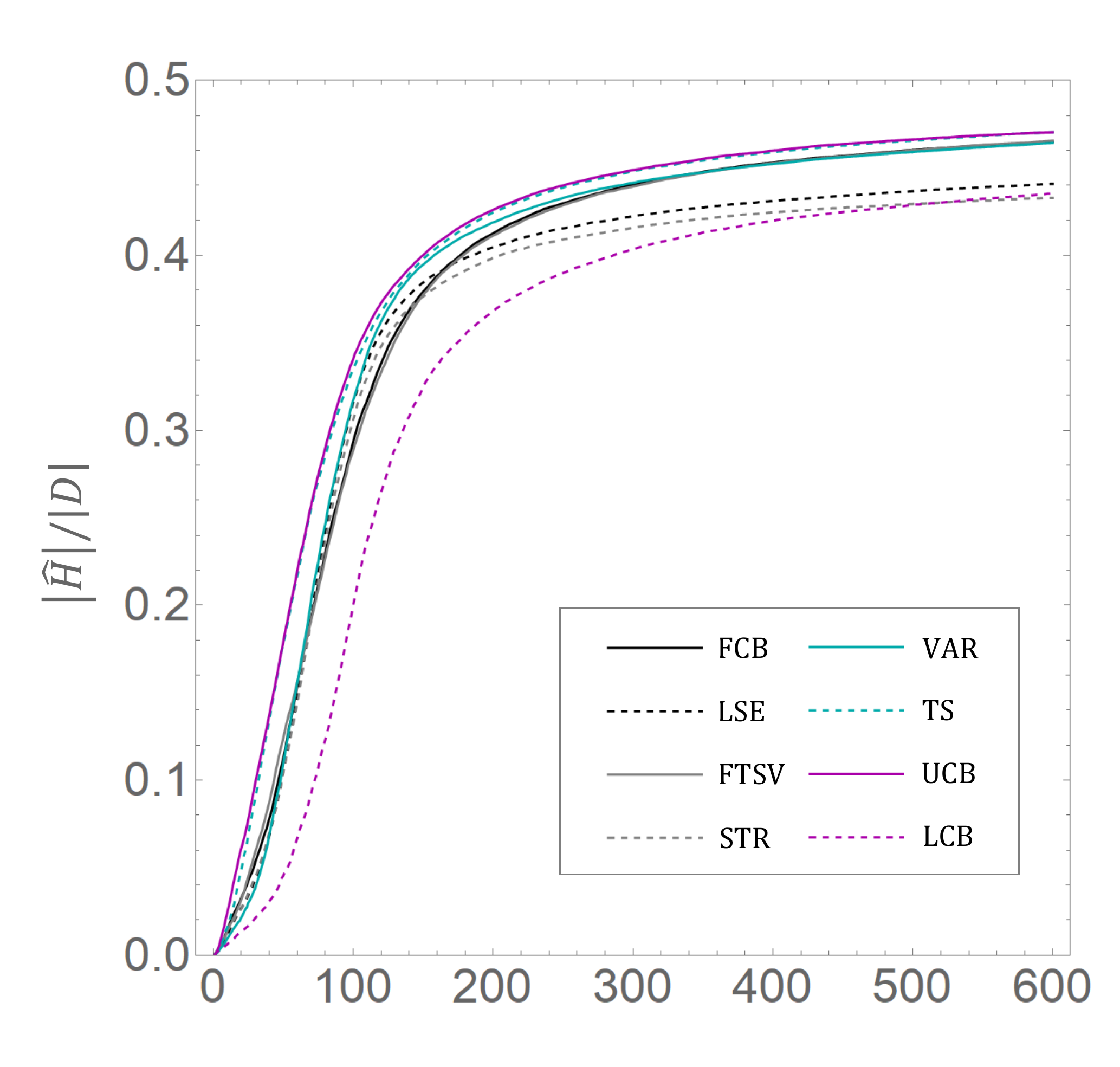} &
\includegraphics[width=0.5\textwidth]{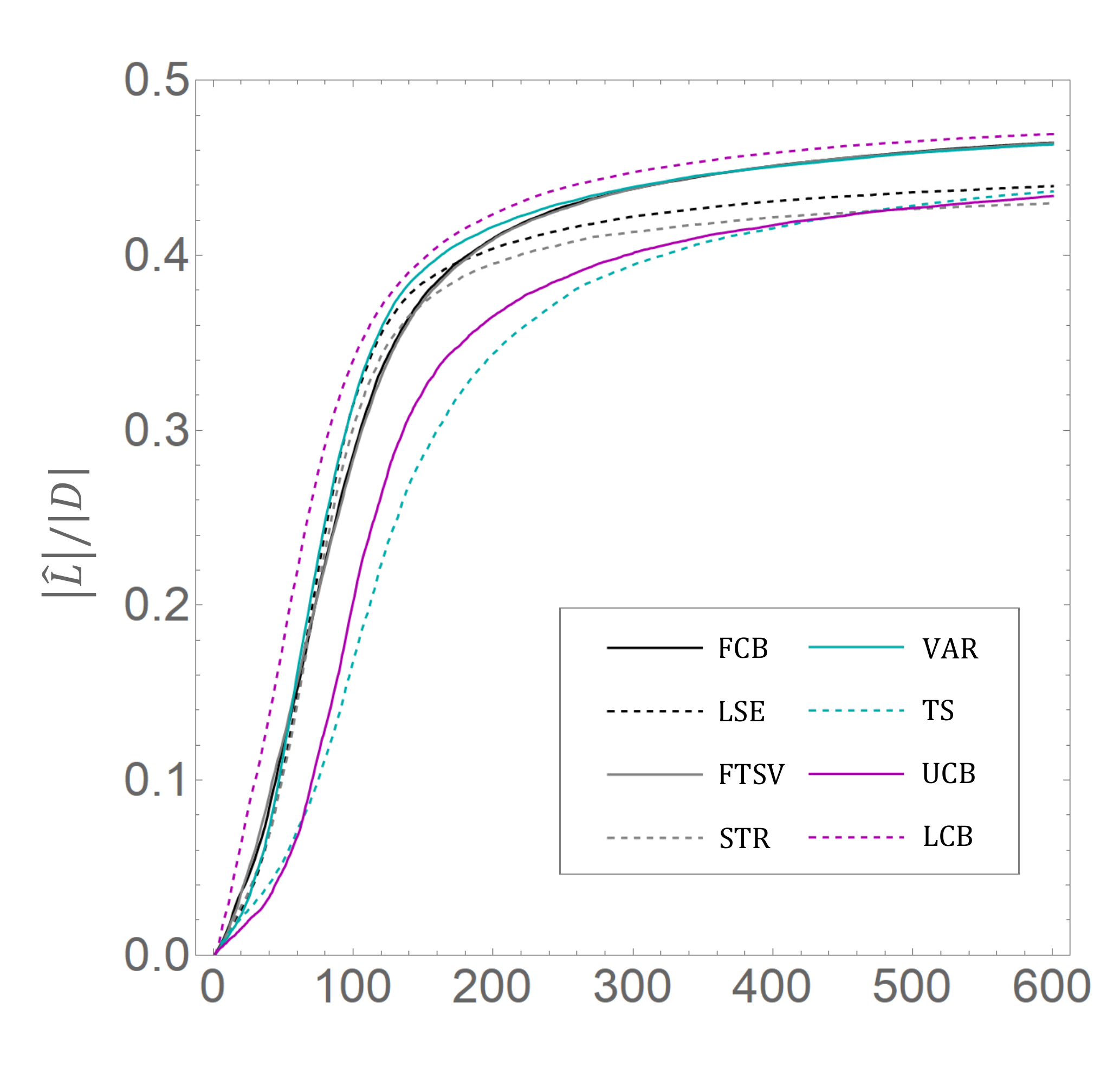} 
\end{tabular}
\caption{Growth curves of $|\hat{H}|/|D|$ (left) and  $|\hat{L}|/|D|$ (right) averaged over $100$ simulations for Algorithm GPCB[ASP] with ASPs \emph{without} rate estimations. The horizontal axis shows the number of observed function values.}\label{fig:GCwoRE}
\end{figure}

We generate target functions $f$ by sampling from the GP prior $\mathrm{GP}(0,k(x,x'))$ that is assumed in our problem settings.
Domain $X$ is defined as $X=[-1,1]^2$.
The parameters $\sigma_{kernel}$ and $l$ of the Gaussian kernel are set as $\sigma_{kernel} = 1$, $l=0.2$.
Values $\varepsilon = 10^{-8}, \sigma=0.1$ are used, and $h$ is set so that $|H_h|/|D|=0.5$ holds in this experiment. 

The average numbers of function values asked by GPCB[ASP] and their 95\% confidence intervals for ASPs without rate estimations are shown in Table~\ref{ex:GPprior}.
Note that the result is averaged over 100 functions generated according to $\mathrm{GP}(0,k(x,x'))$.
TS and UCB  perform well when the true answer is positive but perform poorly in the other case as a result of always selecting the estimated highest point.
LCB  performs oppositely as a result of always selecting the estimated lowest point.
These characteristics of TS, UCB, and LCB can be confirmed by the growth curves of $|\hat{H}|/|D|$ and $|\hat{L}|/|D|$ over the number of asked function values (Figure~\ref{fig:GCwoRE}).
FCB and VAR look the best performers, FTSV performs the next, and LSE follows.
STR does not perform well when $w$ is close to the high-value point rate, and LSE  has a similar tendency.
The reason why VAR performs well without using threshold $h$ in its calculation, is guessed as that VAR's selection resembles FCB's selection
in the GPCB framework, in which all the ASP selects a query point from $U_t$ excluding $\hat{H}$ and $\hat{L}$.

\begin{table}[t]
\centering
  \caption{\#(function values) asked by GPCB[ASP] for ASPs \emph{with} rate estimations. The meanings of the shown numbers are the same as in Table~\ref{ex:GPprior}.}\label{ex:GPprior2}
  \begin{tabular}{lr@{$\pm$}rr@{$\pm$}rr@{$\pm$}rr@{$\pm$}r}
  \hline
    \hline
      ASP                       & \multicolumn{2}{c}{$w=0.2$} & \multicolumn{2}{c}{$w=0.4$} & \multicolumn{2}{c}{$w=0.6$} & \multicolumn{2}{c}{$w=0.8$}  \\ \hline
      FCB$_{\text{RE}}$  & \underline{\textbf{54.63}} & 1.64 & \textit{158.6} & 8.8 & \textit{159.6} & 7.7 & \underline{\textbf{54.32}} & 1.45  \\
      LSE$_{\text{RE}}$  & 62.16 & 2.05 & 216.3 & 27.3 & 222.6 & 42.5 & 63.33 & 1.99  \\
      FTSV$_{\text{RE}}$ & \textit{56.53} & 1.71 & \textit{160.8} & 7.8 & \textit{166.3} & 7.9 & \textit{56.53} & 1.60  \\
      STR$_{\text{RE}}$  & 68.47 & 7.83 & 337.8 & 72.6 & 325.9 & 49.6 & 71.10 & 13.80   \\
      VAR$_{\text{RE}}$  & \textit{56.96} & 1.47 & \underline{\textbf{149.1}} & 8.1 & \textbf{158.4} & 10.9 & \textit{56.88} & 1.89   \\
      \hline\hline
    \end{tabular}
\end{table}
\begin{figure}[t]
\centering
\begin{tabular}{@{}c@{}c@{}}
growth curve over $|\hat{H}|/|D|$ & growth curve over $|\hat{L}|/|D|$\\
\includegraphics[width=0.5\textwidth]{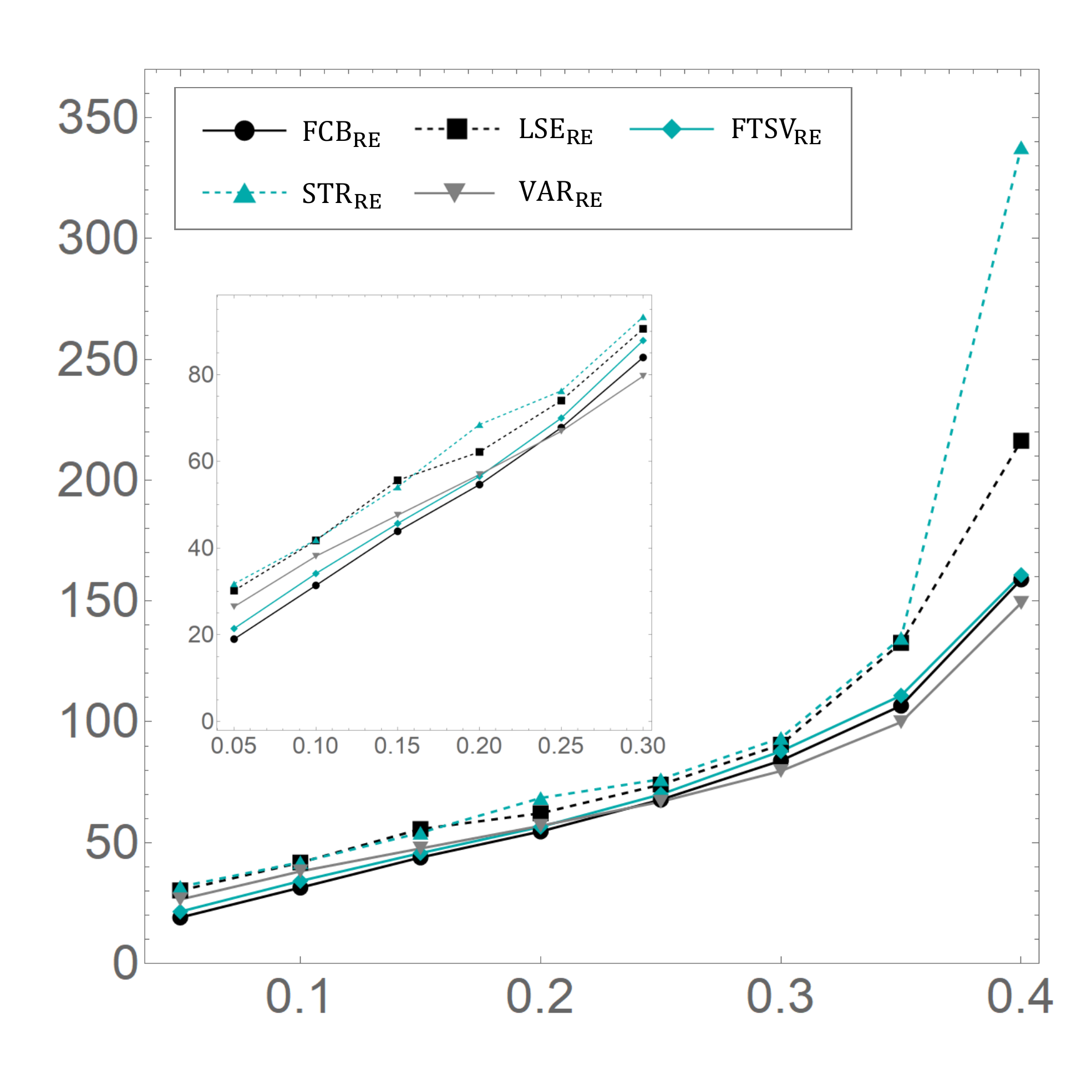} &
\includegraphics[width=0.5\textwidth]{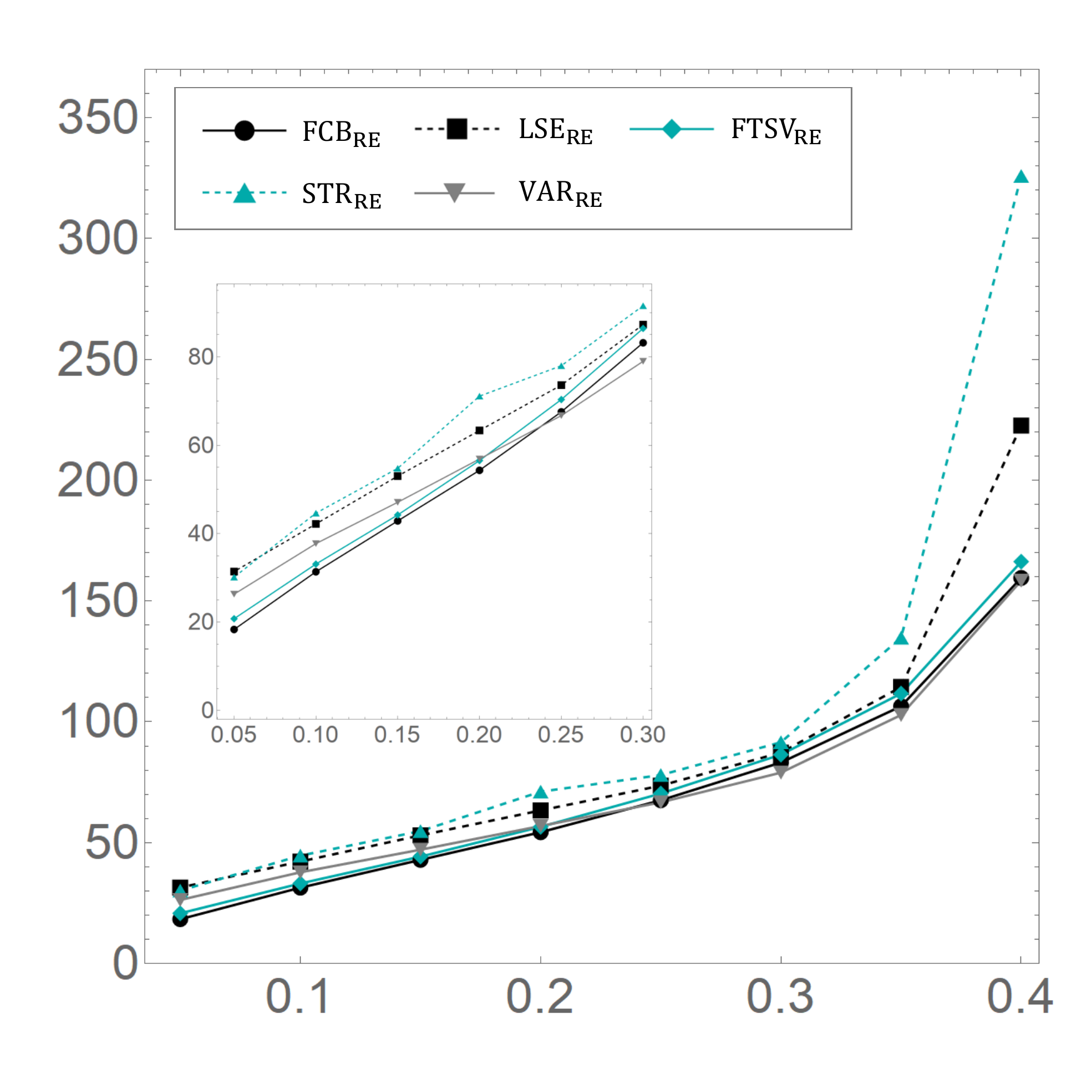} \\
\end{tabular}
\caption{Growth curves of \#(function values) over $|\hat{H}|/|D|$ (left) and $|\hat{L}|/|D|$ (right) for Algorithm GPCB[ASP] with rate estimation versions of ASPs . Each value is averaged over $100$ simulations. The inside plot shows a zoomed portion of a plot  from $0.05$ to $0.30$.}\label{fig:GCwRE}
\end{figure}

\begin{figure}[t]
\centering
\begin{tabular}{@{}c@{}c@{}}
growth curve of $|\hat{H}|/|D|$ & growth curve of $|\hat{L}|/|D|$\\
\includegraphics[width=0.5\textwidth]{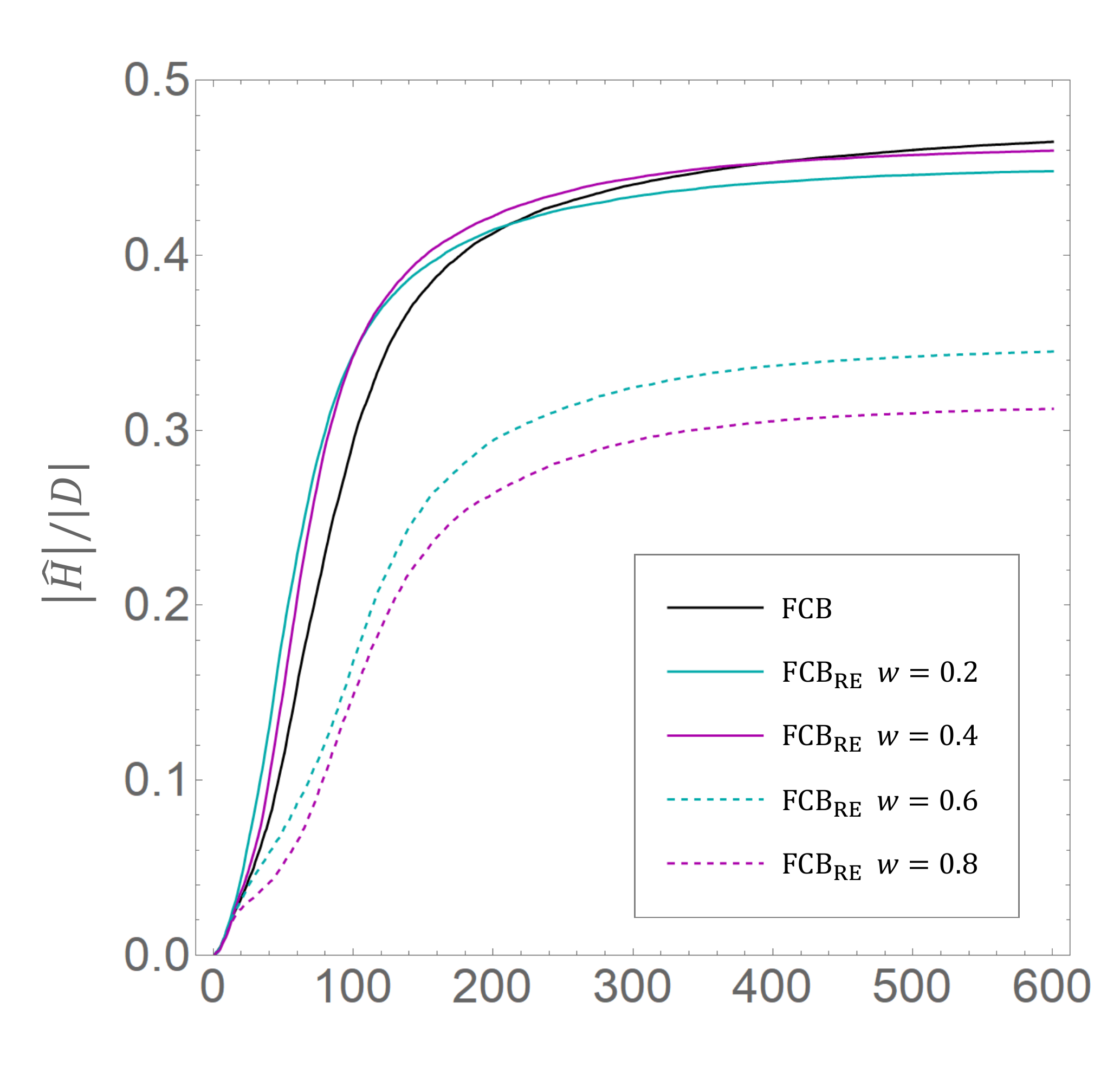} &
\includegraphics[width=0.5\textwidth]{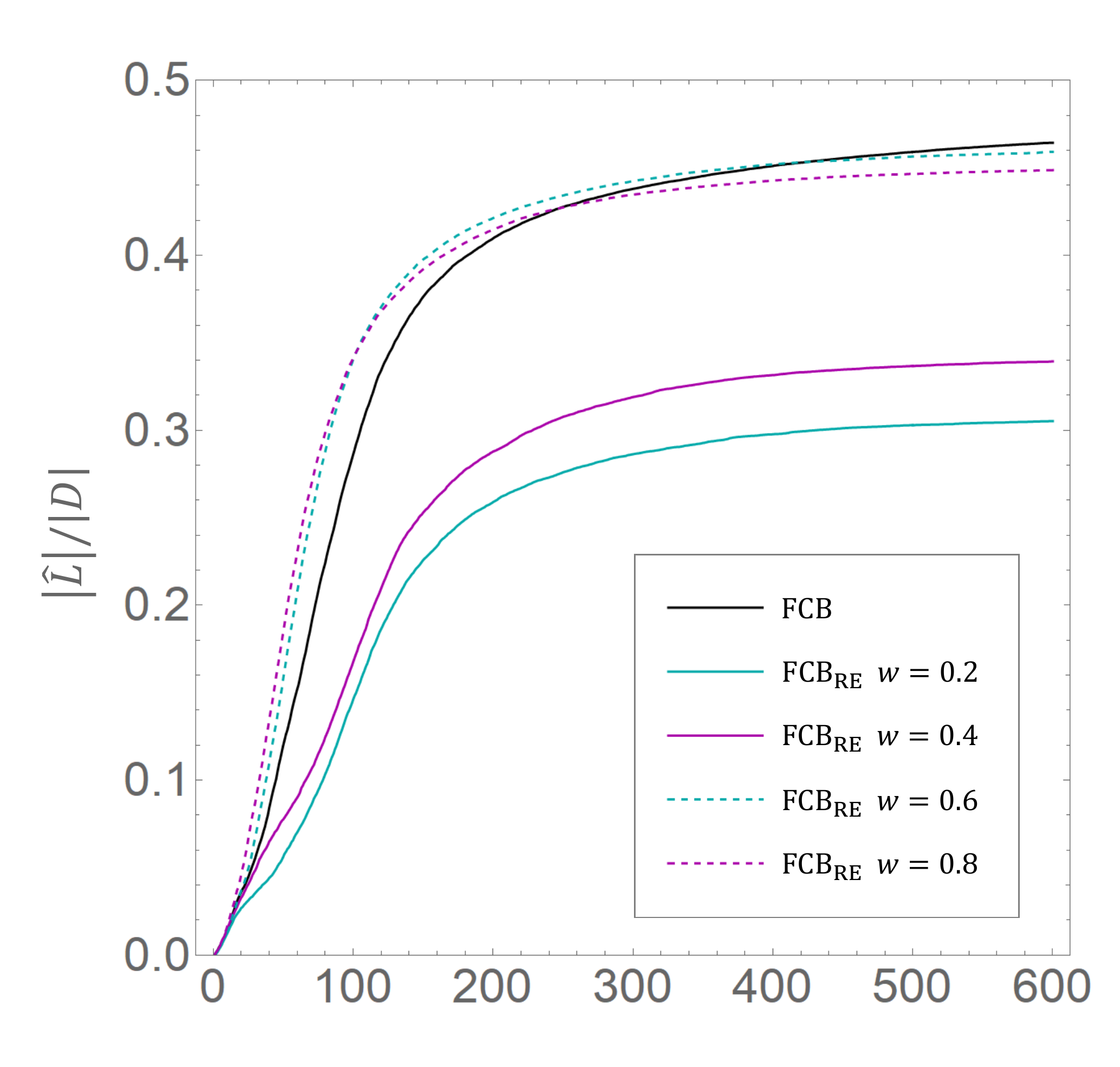} \\
\end{tabular}
\caption{Growth curves of $|\hat{H}|/|D|$ (left) and  $|\hat{L}|/|D|$ (right) averaged over $100$ simulations for Algorithm GPCB[ASP] with ASP=FCB and FCB$_{\text{RE}}$. 
The horizontal axis shows the number of function values.}\label{GWforFCB}
\end{figure}

The results for ASPs with rate estimations are shown in Table~\ref{ex:GPprior2}.
Rate estimations reduce the number of function values asked by GPCB[ASP] by about 20\% on average for FCB, VAR, and FTSV
while  not always reducing that for LSE and STR.
LSE and STR are algorithms for level set estimation whose task is to determine belonging for every point in $D$,
and thus a point $x$ with $f(x)$ close to $h$ is preferred to be asked in some cases to achieve the task efficiently.
Growth curves of the number of asked function values over $|\hat{H}|/|D|$ and $|\hat{L}|/|D|$ in Figure~\ref{fig:GCwRE}
indicate similar behaviors of FCB and FTSV. From the curves, we know that FCB and FTSV perform better than VAR for the thresholds $w$ that are far from the real rate 
$|H_h|/|D|=0.5$ while VAR  performs a little bit better for thresholds $w$ that are close to the real rate.
Growth curves of $|\hat{H}|/|D|$ and  $|\hat{L}|/|D|$ for ASP=FCB and FCB$_{\text{RE}}$ with various $w$ are shown in Figure~\ref{GWforFCB}.
We can see that the setting value of $w$ changes the algorithm's priority on which points in $\hat{H}$ or $\hat{L}$ should be found.

\subsubsection{Benchmark Functions}

In the experiments using two benchmark functions, we compare empirical sample complexities of GPCB[ASP] with ASPs that are good performers for GP-prior-generated functions: FCB, LSE, FTSV, VAR, and their rate estimation versions. Among them, we omit the result for FTSV because its performance was observed to be always similar to (or a little bit worse than) FCB's performance.

\subsubsection*{Gauss Function}

The first benchmark function is the  two-dimensional Gaussian function $f(x,y) = 5 e^{-3(x^2+y^2)}$ for $(x,y)\in X=[-1,1]^2$.
We set $\varepsilon = 10^{-8}, \sigma=0.1$, and $h=0.5$. In this setting, $|H_h|/|D|=0.567118$.
The tuned Gaussian kernel parameters $\sigma_{kernel}$ and $l$ are $1.29$ and $0.53$, respectively.

\begin{table}[tb]
\centering
\caption{\#values of the  two-dimensional Gaussian function asked by GPCB[ASP] averaged over ten simulations and its 95\% confidence interval for various ASPs.}\label{ex:benchgaussf}
  \begin{tabular}{lr@{$\pm$}rr@{$\pm$}rr@{$\pm$}rr@{$\pm$}r}
    \hline\hline
      ASP                       & \multicolumn{2}{c}{$w=0.2$} & \multicolumn{2}{c}{$w=0.4$} & \multicolumn{2}{c}{$w=0.6$} & \multicolumn{2}{c}{$w=0.8$}  \\ \hline
      FCB                       & 10.9 & 0.92 & 24.4 & 1.82 & \textbf{261.0} & 45.6 & 33.3 & 1.07  \\
      LSE                        & 12.3 & 2.05 & 26.0 & 1.58 & 659.9 & 235.4 & \textit{32.1} & 2.53  \\
      VAR                       & 11.5 & 0.51 & 21.5 & 2.22 & \textit{333.3} & 68.4 & \textit{31.1} & 1.60  \\
      FCB$_{\text{RE}}$                 &\textit{7.3} & 0.90 & \textit{17.8} & 0.89 & \textit{292.9} & 25.6 & \textit{30.9} & 2.20  \\
      LSE$_{\text{RE}}$                  & \textit{8.5} & 1.36 & 23.7 & 1.65 & 555.7 & 96.8 & 33.8 & 3.01  \\
      FTSV$_{\text{RE}}$ & \textbf{6.5} & 1.27 & \textit{18.9} & 1.89 & \textit{300.3} & 51.6 & 36.5 & 4.08  \\
      VAR$_{\text{RE}}$                  & \textit{8.4} & 1.76 & \textbf{17.7} & 1.01 & \textit{294.6} & 37.7 & \textbf{28.6} & 1.27  \\
      \hline\hline
    \end{tabular}
\end{table}

The results are shown in Table~\ref{ex:benchgaussf}.
The tendency is the same as in the previous experiment.
FCB and VAR perform better than LSE, and their rate estimation versions perform better.
FCB$_{\text{RE}}$ and VAR$_{\text{RE}}$ are the best performers and FTSV$_{\text{RE}}$ follows.

\subsubsection*{Sinusoidal Function}

The second benchmark function is the sinusoidal function $f(x,y)$ defined as
$f(x, y) = \sin(10x) + \cos(4y) - \cos(3xy)$ for $(x,y) \in X=[0, 2]^2$, which was used in~\citep{Bryan2006}.
We use values $\varepsilon = 10^{-8}, \sigma=0.1$, and $h=0.2$ in this experiment.
In this setting, $|H_h|/|D|=0.447451$.
The tuned Gaussian kernel parameters $\sigma_{kernel}$ and $l$ are $1.78$ and $0.24$, respectively.

\begin{table}[tb]
  \centering
  \caption{\#values of the sinusoidal function asked by GPCB[ASP] averaged over ten simulations and its 95\% confidence interval for various ASPs.}\label{ex:benchsinusof}
  \begin{tabular}{lr@{$\pm$}rr@{$\pm$}rr@{$\pm$}rr@{$\pm$}r}
    \hline\hline
      ASP                       & \multicolumn{2}{c}{$w=0.2$} & \multicolumn{2}{c}{$w=0.4$} & \multicolumn{2}{c}{$w=0.6$} & \multicolumn{2}{c}{$w=0.8$}  \\ \hline
      FCB                       &  83.08 & 0.56 & 265.3 & 3.8 & 102.65 & 0.81 & 56.70 & 0.43 \\
      LSE                        & 72.55 & 0.53 & 662.9 & 34.6 & 99.67 & 1.08 & 60.43 & 0.64\\
      VAR                       & 76.44 & 0.50 & 273.9 & 5.4 & 96.15 & 0.64 & 56.08 & 0.40 \\
      FCB$_{\text{RE}}$                  & 60.52 & 0.79 & \textit{220.6} & 5.3 & \textit{89.27} & 0.82 & \textbf{48.01} & 0.51 \\
      LSE$_{\text{RE}}$                   & 61.33 & 2.15 & 517.0 & 31.1 & 95.53 & 0.81 & 58.13 & 0.53  \\
      FTSV$_{\text{RE}}$  & 65.95 & 0.67 & \textbf{218.4} & 3.1 & 93.64 & 0.66 & \textit{48.61} & 0.62 \\
      VAR$_{\text{RE}}$                   & \textbf{58.46} & 0.63 & 230.8 & 6.0 & \textbf{89.03} & 0.99 & 50.87 & 0.49 \\
      \hline\hline
    \end{tabular}
\end{table}

The results are shown in Table~\ref{ex:benchsinusof}.
The tendency is also almost the same.
Rate estimation reduces the number of samples needed for classification, but its effect on LSE is smaller than the effect on FCB and VAR.
The best performers are FCB$_{\text{RE}}$, FTSV$_{\text{RE}}$, and VAR$_{\text{RE}}$.

\subsection{Real-World Dataset}

\begin{table}[t]
  \centering
  \caption{\#(function values) asked by GPCB[ASP] and TS-CB averaged over ten simulations and their 95\% confidence interval for the cancer index image dataset.}\label{tab:realdata}
  {\small
  \begin{tabular}{lr@{$\pm$}rr@{$\pm$}rr@{$\pm$}rr@{$\pm$}rr@{$\pm$}r}
                \hline\hline
               image                       & \multicolumn{2}{c}{FTC133$_1$}     & \multicolumn{2}{c}{FTC133$_2$}     & \multicolumn{2}{c}{FTC133$_3$}  &  \multicolumn{2}{c}{FTC133$_4$}  & \multicolumn{2}{c}{FTC133$_5$}  \\
{\small $|H_h|/|D|$} & \multicolumn{2}{c}{0.449531}& \multicolumn{2}{c}{0.565354 }& \multicolumn{2}{c}{0.408979}& \multicolumn{2}{c}{0.449792}& \multicolumn{2}{c}{0.522302}\\
\hline
              FCB                        & 747.5 & 7.9& 586.2 & 8.5& 792.2 & 6.2& 744.6 & 9.1& 598.2 & 8.2\\
              LSE                         & 518.8 & 3.1& 512.4 & 3.0& 595.2 & 10.2& 511.9 & 2.8 & 547.4 & 4.2\\
              FTSV                      & 720.3 & 3.2& 543.1 & 3.7& 800.0 & 3.3& 743.8 & 3.4& 587.3 & 4.3\\
              STR                        & 515.7 & 2.4& 513.3 & 2.1& 564.6 & 5.7& 509.6 & 1.9& 531.1 & 3.7\\
              VAR                        & 675.4 & 2.7& 548.6 & 2.3& 748.0 & 3.3& 667.9 & 2.6& 600.7 & 3.5\\              
              FCB$_{\text{RE}}$     & 640.9 & 17.3& 490.3 & 11.7& 776.2 & 7.3& 554.4 & 28.2& 526.8 & 14.9\\
              LSE$_{\text{RE}}$     & \textit{422.3} & 4.8& 428.0 & 4.2& 558.6 & 15.1& \textit{384.6} & 2.1& \textit{439.5} & 4.3\\
              FTSV$_{\text{RE}}$    & 424.8 & 2.4& \textbf{392.8} & 1.9& \textbf{488.1} & 3.3& 389.0 & 1.7& \textit{439.1} & 3.2\\
              STR$_{\text{RE}}$    & \textbf{414.9} & 3.7& 438.7 & 5.4& 529.7 & 10.1& \textbf{383.7} & 2.5& \textit{439.2} & 4.2\\
              VAR$_{\text{RE}}$    & 493.0 & 10.9& 415.1 & 2.6& 687.3 & 7.6& 395.3 & 5.6& \textbf{438.5} & 5.9\\ \hline
              TS-CB                    & 1804.4 & 8.7& 1550.4 & 3.6& 2284.3 & 13.0& 1548.6 & 3.5& 1999.2 & 10.5\\ \hline
		\end{tabular}
  \begin{tabular}{lr@{$\pm$}rr@{$\pm$}rr@{$\pm$}rr@{$\pm$}rr}
                \cline{1-9}
               image                       & \multicolumn{2}{c}{Nthyori31$_1$}     & \multicolumn{2}{c}{Nthyori31$_2$}     & \multicolumn{2}{c}{Nthyori31$_3$}  &  \multicolumn{2}{c}{Nthyori31$_4$}  &  \hspace*{2.065cm} \\
{\small $|H_h|/|D|$} & \multicolumn{2}{c}{0.197938}& \multicolumn{2}{c}{0.236750}& \multicolumn{2}{c}{0.027188}& \multicolumn{2}{c}{0.119469} &\\
\cline{1-9}
              FCB                        & \textbf{819.8} & 2.9& 842.3 & 3.4& \textit{690.7} & 0.9& \textit{766.2} & 2.1 &\\
              LSE                         & 873.8 & 2.2& 893.5 & 2.2& 696.6 & 0.8& 792.9 & 2.0   \\
              FTSV                      & \textit{825.0} & 2.9& 840.4 & 3.1& \textit{691.0} & 0.7& \textbf{762.6} & 1.8 &\\
              STR                        & 872.4 & 2.3& 903.1 & 1.8& 696.3 & 0.9& 793.0 & 2.0 &\\
              VAR                        & 844.1 & 2.5& 885.1 & 2.2& 692.7 & 0.9& 773.0 & 2.0 &\\              
              FCB$_{\text{RE}}$     & \textit{820.0} & 2.7& \textit{837.0} & 3.5& \textit{691.2} & 0.9& \textit{763.3} & 1.9 &\\
              LSE$_{\text{RE}}$     & 868.8 & 2.4& 881.9 & 3.3& 697.6 & 1.0& 792.3 & 2.1 &\\
              FTSV$_{\text{RE}}$    & \textit{822.4} & 2.7& \textbf{832.2} & 2.6& \textbf{690.6} & 0.8& \textit{764.4} & 2.0 &\\
              STR$_{\text{RE}}$    & 871.0 & 2.1& 889.0 & 2.7& 697.0 & 1.1& 792.7 & 2.3 & \\
              VAR$_{\text{RE}}$    & 843.5 & 2.1& 876.9 & 1.8& \textit{692.2} & 0.8& 772.9 &2.0 & \\ \cline{1-9}
              TS-CB                    & 2984.6 & 14.5& 2318.0 & 11.4& 1990.1 & 3.4 & 2782.4 & 9.7 &\\ \cline{1-9}
  \end{tabular}
  }
\end{table}

To check the applicability of Algorithm GPCB[ASP] to fast cancer diagnosis using Raman spectroscopy,
we conduct a simulation using cancer index images that are made from 5 Raman images of cultured human follicular thyroid carcinoma cells (FTC-133)
and 4 Raman images of cultured normal human primary thyroid follicular epithelial cells (Nthy-ori 3-1).
Note that follicular thyroid carcinoma cells are known to be difficult to diagnose from their  fluorescence images to only capture morphological features,
and Raman spectra can detect the difference of constituent molecules between cancer and non-cancer cells.
Raman  spectrum of one point  is composed of intensities of 840  wavenumbers,
and thus  two-dimensional simultaneous measurement is not easy in addition to slow measurement per one point ($1\sim 10$s).
The cancer index of each point is calculated from the  Raman intensities of the point by a decision forest trained using 8 Raman images, excluding  a test image that contains the  illumination point to be tested. See ~\ref{sec:rwdataset} for details on how to construct our cancer index images.
Each cancer index image is composed of $240 \times 400 (= 96,000)$ pixels and each point takes a value in $\{0.00, 0.01, 0.02,\dots,1.00\}$. We shift all the cancer index values by $- 0.5$  to fit them to the GP prior with mean $0$.
We partitioned each whole image into $960$ $(10\times 10)$-grids and let $D$ be the set of $960$ grid centers.
 A noisy function value for a grid center is realized by random sampling from $100$ points of the grid.
We used Gaussian kernel and adjusted the parameters one image by one image in advance by maximum likelihood estimation using 8 other images in a leave-one-out manner. 
Setting thresholds $h=0.0$ (for shifted values) and $w=0.3$, we run GPCB[ASP] for various ASPs and compare their empirical sample complexities.
In order to check the effect of using the correlation assumption, that is, using a GP prior,
we also run the existing multi-armed classification bandit algorithm called ThompsonSampling-CB \citep{Tabata2021}, TS-CB for short, which calculates
the posterior distribution $N(\mu_t(x),\sigma^2_t(x))$ of $f(x)$ for each point $x$ as follows using a normal distribution prior $N(\eta,\tau^2)$:
\begin{align*}
\mu_t(x) &= \frac{n_t(x) \tau^2 \bar{y}_t(x) + \sigma^2 \eta}{n_t(x) \tau^2 + \sigma^2} \text{ and}\\
\sigma_t^2(x) &= \frac{\tau^2 \sigma^2}{n_t(x) \tau^2 + \sigma^2},
\end{align*}
where $n_t(x)$ and $\bar{y}_t(x)$ are the number of selections and the average observed value, respectively, at point $x$ up to $t$.
The next query point $x_t$ of TS-CB is selected in the same way as FTSV$_{\text{RE}}$ except that
$g_t(x)$ is sampled from $N(\mu_{t-1}(x),\sigma_{t-1}^2(x))$ at each point $x$ instead of sampled from $\mathrm{GP}(\mu_{t-1},k_{t-1})$ only once as a function $g_t$ on $D$.
We set $\eta$ and $\tau$ one image by one image, calculating the mean and variance over all the $(10\times 10)$-grids in a leave-one-out manner.

The results are shown in Table~\ref{tab:realdata}.
Note that all 11 algorithms outputted correct answers for all nine images.
The tendency is a little different from the results for synthetic functions.
FCB and FTSV perform better than LSE and STR for non-cancer images but worse for cancer images, and VAR's performances are between them for both types of images. 
The rate estimation works well even for this real-world dataset, especially, FTSV's cancer image performance improves significantly.
As a result, FTSV$_{\text{RE}}$ is the best performer among all the ASPs.
The effect of using the correlation assumption is trivial; TS-CB needs $2.8\sim 4.7$ times larger number of samples than GPCB[ASP] with the best ASP.

\section{Conclusion and Future Work}

In this study, we formulated the problem of Gaussian process classification bandits. It
is the special classification bandit that considers correlations between rewards of arms corresponding to points in $d$-dimensional real space by assuming the GP prior of a reward function.
It is  also a modified level set estimation that enables earlier stopping by changing the objective to classification.
For this problem, we proposed a framework algorithm GPCB[ASP] that can use  various arm selection policies ASPs
and proved a sample complexity for GPCB[FCB] smaller than that for the LSE algorithm, which is an algorithm for the level set estimation. For the level set estimation, the objective-changed GPCB[ASP] with any ASP that always selects an uncertain arm, can be shown to achieve the same sample complexity upper bound as the LSE algorithm.
For the synthetic functions, newly proposed policies FCB and FTSV performed better than policies LSE and STR,
which are the state-of-the-art policies used for the level set estimation.
GPCB[ASP]s using the rate estimation version of ASPs were demonstrated to improve stopping time  
for both synthetic functions and a real-world dataset.
For the real-world dataset (the cancer index image dataset), GPCB[ASP] outperformed the existing classification bandit algorithm TS-CB that does not consider the correlation between arms' rewards, and 
the rate estimation version of FTSV was the best performer. Theoretical analyses of the rate estimation versions are our future work.

\section*{Acknowledgement}
This work was partially supported by JST CREST Grant Number JPMJCR1662 and JSPS KAKENHI Grant Numbers JP18H05413 and JP19H04161, Japan.

\bibliographystyle{elsarticle-harv} 
\bibliography{patrec2022}

\appendix

\section{Asymptotic Behavior of $T(\Delta)$}\label{sec:Tdelta}

In this section, we analyze the asymptotic behavior of function $T(\Delta)$ as $\Delta\rightarrow 0$ for $\beta_t=2\log(|D|\pi^2t^2/6\delta)$ and $\gamma_t=O(d\log t)$ (linear kernel) or  $\gamma_t=O((\log t)^{d+1})$ (squared exponential kernel).
For fixed $|D|$, $\delta$, $d$ and $\sigma^2$, there exist  positive constants $C_{\beta}$ and $C_{\gamma}$ such that $\beta_t\leq C_{\beta}\log t$ and $\gamma_t\leq C_{\gamma}(\log t)^n$  hold for enough large $t$, where $n=1$ for linear kernel and $n=d+1$ for squared exponential kernel.
Define $T'(\Delta)$ as 
\[
T'(\Delta)=\min\left\{t\in \mathbb{N}  \left| \ \frac{t}{C_{\beta}C_{\gamma}(\log t)^{n+1}}> \frac{C_1}{\Delta^2}\right.\right\},
\]
 which serves as an upper bound of $T(\Delta)$.
Consider time $t$ that satisfies equation
\begin{align}
\frac{t}{C_{\beta}C_{\gamma}(\log t)^{n+1}}= \frac{C_1}{\Delta^2} \label{eq:T'},
\end{align}
then $T'(\Delta)=\lceil t\rceil$ holds,
 resulting in $T(\Delta)=O(t)$ for such $t$.
Define  $\alpha(\Delta)$ as
\[
\alpha(\Delta)=\frac{C_1C_{\beta}C_{\gamma}}{\Delta^2},
\]
and, then Eq.~(\ref{eq:T'}) is rewritten as
\[
\frac{t}{(\log t)^{n+1}}= \alpha(\Delta).
\]
By raising both sides to the $1/(n+1)$-th power, we obtain
\[
\frac{t^{\frac{1}{n+1}}}{\log t}= \{\alpha(\Delta)\}^{\frac{1}{n+1}}.
\]
By substituting  $e^{-s(n+1)}$ ($-\infty < s < \infty$) for $t$, this equation is rewritten as
\[
\frac{e^{-s}}{-s(n+1)}=\{\alpha(\Delta)\}^{\frac{1}{n+1}}.
\]
Thus, by arranging this equality, we obtain
\[
se^s=-\frac{1}{(n+1)\{\alpha(\Delta)\}^{\frac{1}{n+1}}}.
\]
Then, $s$ can be represented as
\[
s=W_{-1}\left(-\frac{1}{(n+1)\{\alpha(\Delta)\}^{\frac{1}{n+1}}}\right)
\]
using one branch $W_{-1}$ of the Lambert W function \citep{Lambert2013}.
For function $W_{-1}$, an inequality
\begin{align*}
W_{-1}(-e^{-u-1})>-1-\sqrt{2u}-u 
\end{align*}
is known to hold for $u>0$.
Therefore, for $u$ that satisfies
\begin{align}
e^{-u-1}=\frac{1}{(n+1)\{\alpha(\Delta)\}^{\frac{1}{n+1}}}, \label{eq:u}
\end{align}
an inequality
\[
t=e^{-s(n+1)}<e^{(1+\sqrt{2u}+u)(n+1)}=(e^{(n+1)(1+u)})^{1+\frac{\sqrt{2u}}{1+u}}
\]
holds. Since
\[
e^{(n+1)(1+u)}=(n+1)^{n+1}\alpha(\Delta)
\]
holds by Eq.~(\ref{eq:u}) and $u\rightarrow \infty$ as $\Delta\rightarrow 0$, we have
\[
T(\Delta)<\frac{(C_1C_{\beta}C_{\gamma}(n+1)^{(n+1)})^{(1+\eta/2)}}{\Delta^{2+\eta}}
\]
for any fixed $\eta>0$, that means $T(\Delta)<\frac{C}{\Delta^{2+o(1)}}$ for some constant $C$ that depends on
fixed $|D|$, $\delta$, $d$, and $\sigma^2$.
Taking also $d$ into account, using some constants $C'$ and $C''$ that depend on $|D|$,$\delta$ and $\sigma^2$,
we obtain $T(\Delta)<\left(\frac{C'd}{\Delta^2}\right)^{1+o(1)}$ for the linear kernel
and $T(\Delta)<\left(\frac{C''(d+2)^{(d+2)}}{\Delta^2}\right)^{1+o(1)}$ for the squared exponential kernel.

\section{Proof of Theorem~\ref{coro:wcen}}\label{sec:coro:wcen}

\begin{lemma}\label{lem:hatU_t}
  In Algorithm GPCB[ASP],
  \[
x\in U_t \Rightarrow \sqrt{\beta_t}\sigma_{t-1}(x)>\varepsilon
  \]
holds for any $ t\geq 1 $ and for any arm selection policy function ASP.
\end{lemma}
\begin{proof}
  Assume that $x\in U_t$.
  Then,
  \begin{align*}
  \min(C_t(x))< h-\varepsilon \text{ and }
  \max(C_t(x))\geq   h+\varepsilon
\end{align*}
hold. Thus,
\begin{align*}
\mu_{t-1}(x)-\sqrt{\beta_t}\sigma_{t-1}(x)< h-\varepsilon \text{ and }
\mu_{t-1}(x)+\sqrt{\beta_t}\sigma_{t-1}(x)\geq  h+\varepsilon
\end{align*}
hold.
Therefore,
\begin{align*}
-\mu_{t-1}(x)+\sqrt{\beta_t}\sigma_{t-1}(x)+h> \varepsilon \text{ and  }
\mu_{t-1}(x)+\sqrt{\beta_t}\sigma_{t-1}(x)-h\geq  \varepsilon
\end{align*}
hold.
  By summing both sides, we obtain
\[
2\sqrt{\beta_t}\sigma_{t-1}(x)>2\varepsilon.
\]
\end{proof}

\begin{lemma}[Lemma 5.3 of \citep{gpucb}]\label{APinfogain}
	For $ \mathbf{y}_t = (y_1,\dots,y_t)^\mathsf{T}$ and\\ $\mathbf{f}_t = (f(x_1),\dots,f(x_t))^\mathsf{T}$,
	\[
	I(\mathbf{y}_t; \mathbf{f}_t)  = \frac{1}{2}\sum_{i=1}^{t}\log(1+\sigma^{-2}\sigma_{i-1}^{2}(x_i)).
	\]
\end{lemma}

\begin{lemma}\label{lem:asp_epsilon}
	In Algorithm GPCB[ASP] with any arm selection policy function ASP that selects $x\in U_t$ at any time $t$,
	\[
	2\varepsilon\leq\sqrt{\frac{C_1\beta_t\gamma_t}{t}}
	\]
	holds for $t \geq 1 $ and $ C_1 = 8/\log(1+\sigma^{-2})$ if $U_t\neq \emptyset$.
\end{lemma}
\begin{proof}
  Since $U_t\neq \emptyset$, $U_i\neq\emptyset$ holds for $1\leq i \leq t$, which means 
  $x_i\in U_i$ for $1\leq i\leq t$. Then, an inequality $\varepsilon\leq \sqrt{\beta_i}\sigma_{i-1}(x_i)$ holds by Lemma~\ref{lem:hatU_t}. 
	Thus, similarly to Lemma 5.4 in \citep{gpucb}, from this inequality and an inequality $\sigma_{i-1}^2(x_i)\leq k(x_i,x_i)\leq 1$, it follows that for any $ i\geq 1 $
	\begin{eqnarray*}
		\varepsilon^2&\leq&\beta_i\sigma_{i-1}^{2}(x_i)\\
		&=&\beta_i\sigma^2(\sigma^{-2}\sigma_{i-1}^{2}(x_i))\\
		&\leq&\beta_i\sigma^2C_2\log(1+\sigma^{-2}\sigma_{i-1}^{2}(x_i)),
	\end{eqnarray*}
	where $ C_2 = \sigma^{-2}/\log(1+\sigma^{-2}) $. Thus, we get for any $ t\geq 1 $,
	\begin{eqnarray*}
		C_1\beta_t\gamma_t&\geq&C_1\beta_tI(\mathbf{y}_t;\mathbf{f}_t)  \text{\ \  (by definition of  $\gamma_t$ )}\\
		&=&4\sigma^2C_2\beta_t\sum_{i=1}^{t}\log(1+\sigma^{-2}\sigma_{i-1}^{2}(x_i)) \text{ (by Lemma~\ref{APinfogain})}\\
		&\geq&\sum_{i=1}^{t}4\sigma^2C_2\beta_i\log(1+\sigma^{-2}\sigma_{i-1}^{2}(x_i))\\
		&\geq&\sum_{i=1}^{t}4\varepsilon^2\\
		&=&4t\varepsilon^2.
	\end{eqnarray*}
\end{proof}

\noindent
\textbf{Proof of Theorem~\ref{coro:wcen}}
We prove that Algorithm  GPCB[ASP] terminates after receiving at most $T_1=T(2\epsilon)$ function values for any arm selection policy function ASP.
Assume that Algorithm GPCB[ASP] does not terminate after receiving $T_1$ function values.
	Then $U_t\neq\emptyset$ holds, thus by Lemma~\ref{lem:asp_epsilon},
	\[
	4\varepsilon^2\leq\frac{C_1 \hbox{} \beta_{T_1} \gamma_{T_1} }{T_1}
	\]
         should hold, but this contradicts the inequality that is derived from the definition of $T_1=T(2\varepsilon)$, that is,
        \[
	\frac{T_1}{\beta_{T_1}\gamma_{T_1}}> \frac{C_1}{4\epsilon^2}.
	\]
        Therefore, Algorithm GPCB[ASP] terminates after receiving at most $T_1$ function values. \qed

\section{Proof of Theorem~\ref{th:fcben}}\label{sec:SCUFCB}

In the following two lemmas, we let $Q_t(x)$ denote $[\mu_{t-1}(x)-\sqrt{\beta_t}\sigma_{t-1}(x),\mu_{t-1}(x)+\sqrt{\beta_t}\sigma_{t-1}(x)]$.
Note that $C_t(x)$ can be calculated as $C_t(x)=C_{t-1}(x)\cap Q_t(x)$.

\begin{lemma}\label{APbound_a_t}
	In Algorithm GPCB[FCB],  the following holds for any $ t\geq 1 $,
	\[
         a_t(x_t)\leq 2\sqrt{\beta_t}\sigma_{t-1}(x_t)-\varepsilon.
	\]
\end{lemma}
\begin{proof}
	Note that
	\begin{align*}
	\max\{\max(C_t(x_t))-h, h-\min(C_t(x_t))\}
        \leq\max(C_t(x_t))-\min(C_t(x_t))-\epsilon
	\end{align*}
	holds because $\min\{\max(C_t(x))-h, h-\min(C_t(x))\}\geq \epsilon$ for $x\in U_t$.
	Since
        \[
        a_t(x_t)=\max\{\max(C_t(x_t))-h, h-\min(C_t(x_t))\}
        \]
        for ASP$=$FCB, we have $a_t(\hbox{} x_t)\geq \epsilon$ and 
	\begin{align*}
		a_t(x_t)\leq&\max(C_t(x_t))-\min(C_t(x_t))-\epsilon\\
		\leq&\max(Q_t(x_t))-\min(Q_t(x_t))-\epsilon\\
		=&2\sqrt{\beta_t}\sigma_{t-1}(x_t)-\epsilon.
	\end{align*}
\end{proof}

\begin{lemma}\label{APnonincreasing}
	In Algorithm GPCB[FCB], $a_t(x_t)$ is nonincreasing in $t$.
\end{lemma}
\begin{proof}
 By the  definition of $a_t$, 
	\begin{eqnarray*}
		a_t(x)&=&\max\{\max(C_t(x))-h, h-\min(C_t(x))\}\\
		&=&\max\{\max(C_{t-1}(x)\cap Q_t(x))-h,h-\min(C_{t-1}(x)\cap Q_t(x))\}\\
		&\leq&\max\{\max(C_{t-1}(x))-h, h-\min(C_{t-1}(x))\}\\
		&=&a_{t-1}(x)
	\end{eqnarray*}
	for $x\in U_t$ and $t\geq 1$. Therefore,
	\begin{align*}
	a_t(x_t)=\max_{x\in U_t}a_t(x)
	\leq\max_{x\in U_t}a_{t-1}(x)
	\leq\max_{x\in U_{t-1}}a_{t-1}(x)
	=a_{t-1}(x_{t-1}),
	\end{align*}
	because $U_{t}\subseteq U_{t-1}$. 
\end{proof}

\begin{lemma}\label{APbound_at2}
	In Algorithm GPCB[FCB],
	\[
	a_t(x_t)+\epsilon\leq\sqrt{\frac{C_1\beta_t\gamma_t}{t}}
	\]
	holds for $t \geq 1 $ and $ C_1 = 8/\log(1+\sigma^{-2})$.
\end{lemma}
\begin{proof}
	Similarly to Lemma 5.4 in \citep{gpucb}, from Lemma~\ref{APbound_a_t} and $\sigma_{i-1}^2(x_i)\leq k(x_i,x_i)\leq 1$ it follows that for any $ i\geq1 $
	\begin{eqnarray*}
		(a_i(x_i)+\epsilon)^2&\leq&4\beta_i\sigma_{i-1}^{2}(x_i)\\
		&=&4\beta_i\sigma^2(\sigma^{-2}\sigma_{i-1}^{2}(x_i))\\
		&\leq&4\beta_i\sigma^2C_2\log(1+\sigma^{-2}\sigma_{i-1}^{2}(x_i)),
	\end{eqnarray*}
	where $ C_2 = \sigma^{-2}/\log(1+\sigma^{-2}) $. Thus, we get for any $ t\geq 1 $,
	\begin{eqnarray*}
		C_1\beta_t\gamma_t&\geq&C_1\beta_tI(\mathbf{y}_t;\mathbf{f}_t)  \text{\ \  (by definition of  $\gamma_t$ )}\\
		&=&4\sigma^2C_2\beta_t\sum_{i=1}^{t}\log(1+\sigma^{-2}\sigma_{i-1}^{2}(x_i)) \text{\ \  (by Lemma~\ref{APinfogain})}\\
		&\geq&\sum_{i=1}^{t}4\sigma^2C_2\beta_i\log(1+\sigma^{-2}\sigma_{i-1}^{2}(x_i))\\
		&\geq&\sum_{i=1}^{t}(a_i(x_i)+\epsilon)^2\\
		&\geq&\frac{1}{t}\left(\sum_{i=1}^{t}(a_i(x_i)+\epsilon)\right)^2  \text{\ \  (by Cauchy-Schwarz)}\\
		&\geq&\frac{1}{t}(t(a_t(x_t)+\epsilon))^2  \text{\ \  (by Lemma~\ref{APnonincreasing})}\\
		&\geq&t(a_t(x_t)+\epsilon)^2.
	\end{eqnarray*}
\end{proof}

\noindent
\textbf{Proof of Theorem~\ref{th:fcben}}
By Corollary~\ref{APcorrect-assumptionen}, with probability at least $1-\delta$
\begin{equation}
  f(x)\in C_t(x) \text{ for any } x\in D \text{ and } t\geq 1 \label{APassumption}
  \end{equation}
holds. Therefore, we only have to prove that  GPCB[FCB] terminates after receiving $T_1=T(\Delta_w+\varepsilon)$ function values when Condition (\ref{APassumption}) holds.
        Consider the case with $|H_{h+\varepsilon}|/|D|\geq w$.
	Assume that GPCB[FCB] does not terminate after receiving $T_1$ function values.
	Since GPCB does not terminate at $t=T_1$, there exists $x\in H_{h+\varepsilon}\cap U_{T_1}$ with $f(x)-h\geq \Delta_w$.
	For such $x$,
\[
	a_{T_1}(x_{T_1})\geq a_{T_1}(x)\geq \max(C_{T_1}(x))-h\geq f(x)-h\geq \Delta_w 
\]
	by Condition~(\ref{APassumption}). Thus, by Lemma~\ref{APbound_at2}, we have
	\[
	(\Delta_w+\varepsilon)^2\leq \frac{C_1\beta_{T_1}\gamma_{T_1}}{T_1},
	\]
	which contradicts the inequality that is derived from the definition of $T_1=T(\Delta_w+\varepsilon)$, that is,
	\[
	\frac{T_1}{\beta_{T_1}\gamma_{T_1}}> \frac{C_1}{(\Delta_w+\varepsilon)^2}.
	\]
        Therefore, GPCB[FCB] terminates after at most $T_1$ iterations.
        We can similarly prove the theorem statement in the case with $|L_{h-\varepsilon}|/|D|> 1-w$.   \qed

\section{Cancer Index Image Dataset}\label{sec:rwdataset}

 This section describes the details of the real-world dataset that we use in our experiment.

\subsection{Dataset}

The dataset contains 9 cancer-index images of follicular thyroid cells.
Five of them are images of FTC-133 (human follicular thyroid carcinoma cells),  
and the rest four are images of Nthy-ori 3-1 (human follicular thyroid epithelial cells).
The size of each image is 240 $\times$ 400 pixels and each pixel has a cancer index value,
which is calculated by the \emph{cancer index function} that maps from \emph{Raman spectra} of the point to a real value in $\{0.01n\mid n=0,1,\dots,100\}$.

\subsubsection{Raman Spectra and Raman Image}

 A \emph{Raman spectrum} at one pixel consists of Raman scattering intensities at different 840 wavenumbers ranging from 542 cm$^{-1}$ to 3087 cm$^{-1}$.
A \emph{Raman image} is an image of 240 $\times$ 400 pixels whose pixel has 840  Raman intensities.
The spectral processing procedure described in ~\ref{spectralprocessing} is applied as in the literatures \citep{zoladek2010label,horiue2020raman}. 

\subsubsection{Cancer Index Function}

 We use a random forest classifier with 100 decision trees as a cancer index function.
The cancer index is the ratio of cancer voters among 100 component decision trees.
For the pixels in each Raman image, the cancer index function used to calculate their cancer indices is trained in   a leave-one-out manner
using the other 8 Raman images whose pixel is labeled ``cancer'', ``non-cancer'' or ``background''.
The pixel Labeling is done by the process described in ~\ref{trainingdata}.

\subsection{Pixel Labeling}\label{trainingdata}

In an image that is composed of cell and background regions, spectra from cells typically show high intensity at high-wavenumber region, or the region from 2800 $\mbox{cm}^{-1}$ to 3000 $\mbox{cm}^{-1}$. 
 We used a clustering algorithm to spatially separate the cell region and  background region based on this fact. 
We segmented image data into 3 regions by applying a 3-means clustering algorithm for the spectra of the high-wavenumber region. We named these regions  ``cell'', ``marginal'', and ``background'' in the decreasing order of the mean intensity at the high-wavenumber region. 
The spectra in the ``cell'' region in FTC-133 (Nthy-ori 3-1) image are labeled ``cancer'' (``noncancer'') and spectra in the ``background'' region are labeled ``background'' regardless of whether it is in the FTC-133 image data or the Nthy-ori 3-1 image data. 

\subsection{Raman Spectral Processing Procedure}\label{spectralprocessing}

First, we detect the spectra affected by cosmic  rays and ignore them. 
Very high intensities are detected in the small spot in a CCD sensor when the cosmic rays pass through the CCD sensor. We exclude the cosmic-ray-influenced spectra for learning since this is not a signal from a sample. 
We consider the spectra with very high intensity (larger than mean intensity + 10 $\times$ standard deviation) at any  wavenumber to be affected by a cosmic ray since these spectra  have very high intensity at some  wavenumber. 
To increase the S/N ratio of the image data, singular-value decomposition (SVD) is performed for the spectra. 
We reconstruct the image data using the top 30 significant components of SVD.
After that, for fluorescence subtraction to correct the  baseline of spectra,  the recursive polynomial fitting algorithm \citep{lieber2003automated} of six degrees is performed. 
Finally, the spectra are normalized so that the total intensity is 1. 

\end{document}